\documentclass[12pt]{article}
\usepackage{enumerate}
\usepackage{amssymb}
\usepackage{amsmath}
\usepackage{mathrsfs}
\usepackage{graphicx}
\usepackage{subfigure}
\usepackage{color}
\usepackage{amsthm}
\usepackage{algorithmic}

\date{}

\usepackage{geometry}
\newtheorem{thm}{Theorem}

\theoremstyle{remark}
\newtheorem{remark}{Remark}
\theoremstyle{definition}

\newtheorem{assum}{Assumption}

\usepackage{authblk}
\usepackage{multirow}
\usepackage{mathtools}

\usepackage[ruled,vlined]{algorithm2e}

\begin{document}

\title{An Actor-Critic Algorithm with Function Approximation for Risk Sensitive Cost Markov Decision Processes}

\author[1]{Soumyajit Guin}
\author[2]{Vivek S.\ Borkar}
\author[1]{Shalabh Bhatnagar}

\affil[1]{Department of Computer Science and Automation, Indian Institute of Science, Bengaluru 560012, India\footnote{Email: gsoumyajit@iisc.ac.in, shalabh@iisc.ac.in. SB was supported by the J.C.Bose Fellowship of the Science and Engineering Research Board, Project No.~DFTM/02/3125/M/04/AIR-04 from DRDO under DIA-RCOE, the Walmart Centre for Tech Excellence, IISc and the Robert Bosch Centre for Cyber Physical Systems, IISc.}}
\affil[2]{Department of Electrical Engineering, Indian Institute of Technology Bombay, Mumbai 400076, India\footnote{Email: borkar.vs@gmail.com. VSB was initially supported by S.\ S.\ Bhatnagar Fellowship from the Council of Scientific and Industrial Research, Government of India, and subsequently by a grant from Google Research.}}
\maketitle

\begin{abstract}
In this paper, we consider the risk-sensitive cost criterion with exponentiated costs for Markov decision processes and develop a model-free policy gradient algorithm in this setting. Unlike additive cost criteria such as average or discounted cost, the risk-sensitive cost criterion is less studied due to the complexity resulting from the multiplicative structure of the resulting Bellman equation. We develop an actor-critic algorithm with function approximation in this setting and provide its asymptotic convergence analysis. We also show the results of numerical experiments that demonstrate the superiority in performance of our algorithm over other recent algorithms in the literature.
\end{abstract}

\section{Introduction}
\label{introduction}

In this section, we provide an introduction to risk sensitive control based reinforcement learning algorithms as well as discuss some related work along these lines.

\subsection{Introduction to Risk-Sensitive Control Based Algorithms}
Reinforcement Learning (RL) is a machine learning paradigm that allows an agent to learn to take optimal actions in an uncertain environment, see for instance, \cite{SuttonBarto2018}, \cite{NDPbook} for text-book treatments. Traditionally, the goal in RL is to minimize the expected summation of costs or alternatively, to maximize the expected summation of rewards. In this paper, however, we will look into the risk-sensitive cost criterion, where we aim to minimize the expected product of costs. This generalizes the classical mean-variance optimization formulation, as is easy to see by expanding the risk-sensitive cost using a Taylor series expansion up to the quadratic term. Unlike the mean-variance formulation, it allows for a dynamic programming based solution. This has been an early motivation to study this problem.
This very aspect also helps us develop critic schemes for evaluating a policy, where we can use the current estimate of the value function to bootstrap.

In Section~\ref{framework}, we present the basic framework of Markov decision processes (MDP) under the risk sensitive cost criterion and recall the  policy gradient formula for the same from \cite{BorkarD}. This will be followed by Section~\ref{acarc} with fully justified critic scheme and actor schemes. We then provide the convergence analysis of our algorithm in Section~\ref{convergence}. Next, in Section~\ref{numerical}, we provide numerical experiments to show that our algorithm successfully finds policies with low standard deviation and is more robust than other algorithms in the literature. By ‘robust’ we mean that the worst case
cumulative cost is better managed. This is because the risk-sensitive
criterion minimizes the maximum expected growth rate of the exponentiated cost.
 Finally, Section~\ref{conclusion} presents the concluding remarks and points to certain interesting future directions. 

\subsection{Other Related Work}

We survey here some prior work related to the development of algorithms for risk-sensitive cost control in the case of exponentiated costs. Model-based value iteration and policy iteration methods for risk-sensitive MDPs are given in \cite{BorkarM, M2}. \cite{murthy2023} proves the convergence of modified policy iteration (MPI) in risk sensitive MDPs, where MPI is a generalization of both value and policy iteration. An algorithm for estimating the value function for the evaluation of policies using linear function approximation is provided in \cite{BasuA}. 
Amongst (model-free) reinforcement learning algorithms, there has been prior work on developing algorithms in the lookup table case. For instance,  the Q-Learning algorithm for risk-sensitive cost control in the lookup table setting is proposed and analyzed in \cite{BorkarC}. A tabular actor-critic algorithm for risk-sensitive cost is proposed and analyzed  in \cite{BorkarD} where a policy gradient formula is obtained. In \cite{JMLR}, the problem of long-run average cost control in MDPs conditioned on rare events is reformulated using the risk-sensitive cost criterion and an actor-critic algorithm is analyzed where the critic runs a temporal difference recursion for the risk sensitive cost and the actor incorporates a zeroth order gradient estimation scheme from \cite{Bhatnagaretal, BhatnagarBook}. This algorithm is again for the lookup table case. In \cite{nass2019, fei2021} learning algorithms for risk-sensitive cost have been presented in the episodic setting. We mention here that there is a vast body of literature on risk sensitive measures. Exponentiated risk sensitive cost measures are also popular as entropic risk measures. It is shown in \cite{kupper} that the only dynamic risk measures that are law invariant and time consistent are the entropic risk measures though they are not coherent, see \cite{artzner}. An overview of model-based and model-free algorithms in various settings including for risk sensitive cost criterion is provided in \cite{BorkarBridge}.

Unlike the above algorithms, our algorithm uses function approximation which is essential for large state and action settings. A policy gradient algorithm for risk-sensitive cost under function approximation has recently been proposed in \cite{moharrami2024}, where the authors update the policy using Monte-Carlo estimates of the risk-sensitive cost obtained from a sample trajectory. The sample path starts in a certain prescribed state that is positive recurrent under all policies and ends with that state. This leads to high variance in the estimation of the cost, rendering the algorithm potentially infeasible for large state spaces. In contrast, our algorithm is an incremental update, albeit multi-timescale procedure that performs updates each time a data sample becomes available. Our algorithm learns the critic values that are functions of the state, and uses these to update the policy. As seen from the experiments in Section~\ref{numerical}, this results in lower standard deviation and makes the algorithm more robust.

\subsection{Our Contributions}

We now summarize our main contributions in this paper.
\begin{itemize}
	\item We present a novel incremental-update actor-critic algorithm for the risk-sensitive cost criterion.
	\item We present an asymptotic convergence analysis of the proposed algorithm.
	\item We present numerical experiments that demonstrate that the algorithm is more robust than other algorithms in the literature.
\end{itemize}

\section{Problem Framework}
\label{framework}

By a Markov decision process (MDP), we mean a sequence $X_n,n\geq 0$ of random variables taking values in  a finite state space $S$ 
and satisfying
\[P(X_{n+1}=j\mid X_k, Z_k, k\leq n) = p(X_n, Z_n,j), \ \forall n\geq 0,\ \mbox{ a.s.}
\]
Here $Z_n$ is the action chosen at time $n$ when the state is $X_n$ and  $Z_n \in A$, the finite set of actions. For ease of exposition, we assume that every action in $A$ is feasible in every state $i\in S$. 
Also, $p(i,a,j)$ denotes the transition probability from state $i$ to $j$ when action $a$ is chosen in state $i$. If $Z_n = f(X_n)$, $\forall n\geq 0$, for some function $\displaystyle f:S\rightarrow A$, then we call $f$ a stationary deterministic policy (SDP). Similarly, if $Z_n$ is sampled from a distribution $\pi(X_n)$ for some $\pi: S\rightarrow \mathcal{P}(A)$ $\forall n$, then $\pi$ is called a stationary randomized policy (SRP).
Here, $\mathcal{P}(A)$ denotes the space of probability vectors on $A$. In what follows, we shall work with SRP. We let $\pi(i) = (\pi(i,a),a\in A)^T,i \in S$ denote an SRP and $c: S\times A\times S\rightarrow \mathbb{R}$, the single-stage cost function. Our goal is to find the optimal policy which minimizes the risk-sensitive cost given by
\begin{equation}
	J_{\pi}=\limsup_{n \uparrow \infty} \frac{1}{n} \log \mathbb{E} \big[e^{\sum_{m=0}^{n-1}\alpha c(X_m,Z_m,X_{m+1})} \big],
\end{equation}
where $\alpha$ is the risk-sensitivity parameter.
\begin{assum}\label{assum1}
Under every stationary randomized policy $\pi$, the Markov chain is irreducible and aperiodic.
\end{assum}

Let $P_{\pi}$ be the probability transition matrix under policy $\pi$ with entries $P_{\pi}(i,j)=\sum_{a \in A}\pi(i,a)p(i,a,j)$. By Assumption~\ref{assum1}, $P_{\pi}$ is irreducible. Let $d_{\pi}(i),i \in S$ be the associated stationary distribution. For any policy $\pi$, we consider the matrix $Q_{\pi}$, whose $(i,j)$-th element is 
\[
\sum_{a \in A} \pi(i,a)e^{\alpha c(i,a,j)}p(i,a,j).
\]
By Assumption~\ref{assum1}, the above matrix is irreducible and aperiodic. Therefore by Perron-Frobenius theorem, it has a unique eigenvalue-eigenvector pair $(\lambda_{\pi},V_{\pi})$, with $\lambda_{\pi}>0$ and $V_{\pi}>0$ componentwise, satisfying
\begin{equation}\label{egeq}
\lambda_{\pi}V_{\pi}(i)=\sum_{a \in A}\pi(i,a)\sum_{j \in S} e^{\alpha c(i,a,j)}p(i,a,j)V_{\pi}(j),
\end{equation}
{rendered unique by setting $V_\pi(i_0)=1$ for some fixed $i_0\in S$.}

Define another transition probability  function by: for $i,j \in S$:
\begin{equation}
\tilde{p}_{\pi}(i,j)=\frac{\sum_{a \in A}\pi(i,a)e^{\alpha c(i,a,j)}p(i,a,j)V_{\pi}(j)}{V_{\pi}(i)\lambda_{\pi}}.
\end{equation}
It can be seen that $\tilde{p}_{\pi}$ is indeed a probability transition function. Denote by $\{\tilde{X}_n\}$ the Markov chain under policy $\pi$ with probability transition function $\tilde{p}_{\pi}$. Then $\{\tilde{X}_n\}$ is also irreducible and aperiodic. Therefore by \cite{BorkarM}, $\log{\lambda_{\pi}}$ is the cost associated with a policy $\pi$. 

Function approximation becomes necessary when the state-action space is  large. We parameterize both the policy and the value function. However, we consider parameterized policies and associated value functions in the look-up table case and later consider the case of function approximation for the value function as well. Let the  policy $\pi$ be  parameterized by $\theta \in \mathbb{R}^{x_1}$. Henceforth we write $\pi_{\theta}$ instead of $\pi$. By an abuse of notation, throughout the paper, we shall use $\pi$ and $\theta$ interchangeably. We now have the following assumption.
\begin{assum}\label{assum2}
$\pi_{\theta}$ is twice continuously differentiable in $\theta$. Moreover $\pi_{\theta}(i,a)>0$ for all $a \in A, i \in S$.
\end{assum}
\begin{remark}
The above assumption can be seen to be satisfied, for example, by the parameterized Gibbs or Boltzmann distribution, see Proposition 11 of \cite{LakshminarayananC}. The commonly used $\epsilon$-greedy policies that randomize among all controls with a small probability provide another example.
\end{remark}

The policy gradient expression in the case when $V_\pi(i)$, $i\in S$ is estimated directly without function approximation, has been derived in \cite{BorkarD}, which we  recall here. The policy gradient provides the gradient search direction in which the parameters should be updated in order to optimize the cost. Let the associated stationary distribution of the Markov chain $\{\tilde{X}_n\}$  with probability transition function $\tilde{p}_{\pi}$ under policy $\pi$ be \(\tilde{d}_{\pi}(i), i\in S\).
We know that the cost under a policy $\pi$ is $\log{\lambda_{\pi}}$. Then the policy gradient expression using a baseline $\zeta(i),i\in S$ is as follows, see \cite{BorkarD}:
\begin{equation}\label{PG}
\begin{split}
&\nabla_{\theta} \log{\lambda_{\pi}}\\
=&\sum_i \tilde{d}_{\pi}(i) \sum_a\nabla_{\theta} \pi(i,a)\\
&\times \Bigg(\frac{\sum_j e^{\alpha c(i,a,j)} p(i,a,j) V_{\pi}(j)}{\lambda_{\pi} V_{\pi}(i)}-\zeta(i)\Bigg)\\
=&\sum_i \tilde{d}_{\pi}(i) \sum_a \pi(i,a) \nabla_{\theta} \log{\pi(i,a)}\\
&\times \Bigg(\frac{\sum_j e^{\alpha c(i,a,j)} p(i,a,j) V_{\pi}(j)}{\lambda_{\pi} V_{\pi}(i)}-\zeta(i)\Bigg).
\end{split}
\end{equation}
An ideal baseline here will be $\zeta(i)=1, \forall i \in S$, which we will be using henceforth. In what follows, we use linear function approximation for $V_\pi(i), i\in S$, as follows: $V_{\pi}(i) \approx r^T\phi(i), i \in S$ where $r \in \mathbb{R}^{x_2}$  is a parameter vector  and $\phi(i) \in \mathbb{R}^{x_2}$ is the feature vector associated with state $i$. Let $\Phi \in \mathbb{R}^{|S|\times x_2}$ be the feature matrix with rows $\phi(i)^T$. Let the $(i,k)$-th entry of $\Phi$ be $\phi^k(i)$. We now make the following assumption analogous to \cite{BasuA}:
\begin{assum}\label{assum3}
The columns of $\Phi$ are orthogonal vectors in the positive cone of $\mathbb{R}^{|S|}$. Under every policy $\pi$, the submatrix of $P_{\pi}$ corresponding to $\cup_k \{i:\phi^k(i)>0\}$ is irreducible.
\end{assum} 

\section{Actor Critic Algorithm for Risk-sensitive cost}
\label{acarc}
We propose a three-timescale stochastic approximation algorithm for our problem. This allows us to treat all recursions separately in the convergence analysis, see Chapter 8 of \cite{BorkarB} for details. Before we proceed further, we make the following assumptions on the step-sizes or learning rates of the coupled iterates. 
\begin{assum}[Step-Sizes]
	\label{a4}
	The step-size sequences $a(n)$, $b(n)$, $c(n)$, $n\geq 0$, satisfy the following conditions:
	\begin{itemize}
        \item[(i)] $a(n),b(n),c(n)>0$, $\forall n$.
		\item[(ii)] ${\displaystyle \sum_n a(n)=\sum_n b(n)=\sum_n c(n) =\infty}$.
		\item[(iii)] ${\displaystyle \sum_n (a(n)^2 + b(n)^2+c(n)^2)<\infty}$. 
		\item[(iv)] ${\displaystyle \frac{b(n)}{a(n)} \rightarrow 0}$, ${\displaystyle \frac{c(n)}{b(n)} \rightarrow 0}$ as $n\rightarrow\infty$.
		\end{itemize}
\end{assum}

We will start with our critic update. Let $i_n$ be the state visited and $z_n$ the action picked at time instant $n$. Let us first consider the problem when the model (cost and probability  transition function) is known, and our goal here is to find $V_{\pi}$ for a given policy $\pi$. Since $V_{\pi}$ is the eigenvector corresponding to the dominant eigenvalue $\lambda_{\pi}$ of the matrix $Q_{\pi}$, we can use the well-known power method to find it. Let us define the operator $T_{\pi}:\mathbb{R}^{|S|} \rightarrow \mathbb{R}^{|S|}$ as below, 
\begin{equation}
T_{\pi}V(i)=\sum_{a \in A}\pi(i,a)\sum_{j \in S}p(i,a,j)e^{\alpha c(i,a,j)}V(j).
\end{equation}
Consider the relative value iteration scheme adapted from \cite{BorkarM}:
\begin{equation}
V_{n+1}(i)=\frac{T_{\pi}V_n(i)}{T_{\pi}V_n(i_0)}, \mbox{ } i\in S, \mbox{ } n\geq 0,
\end{equation}
where $i_0$ is a fixed state. By standard arguments for power iteration,  $V_n \rightarrow V_{\pi}$ and $T_{\pi}V_n(i_0) \rightarrow \lambda_{\pi}$ as $n\rightarrow\infty$. 

In order to motivate the form of our actor-critic algorithm with function approximation that we develop, we first start by looking at the form of an actor-critic algorithm in the tabular setting. We shall subsequently derive the form of the algorithm with function approximation. 
The tabular version of the algorithm when the model is unknown is adapted from the Q-learning algorithm in \cite{BorkarC}, also formally given in \cite{BasuA}. Let $\nu(i_n,n)$ be the number of times that state $i_n$ is visited until the time instant $n$. This quantity helps keep track of the number of times $V(i_n)$ is updated until instant $n$ as we implement here an asynchronous version of the algorithm, see \cite{BorkarAsyn, BorkarB} for a general treatment of asynchronous stochastic approximations. The critic update for the tabular version would correspond to the following: 
\begin{equation}\label{critic-tabular}
\begin{split}
V_{n+1}(i_n)=V_n(i_n)+a(\nu(i_n,n))\\
\Bigg(\frac{e^{\alpha c(i_n,z_n,i_{n+1})}V_n(i_{n+1})}{V_n(i_0)}-V_n(i_n)\Bigg),
\end{split}
\end{equation}
{with $V_{n+1}(j)=V_n(j)$, $\forall j\not= i_n$.}
The tabular iterates $V_n \stackrel{\triangle}{=} (V_n(i),i\in S)^T$ then converge to $V_{\pi} \stackrel{\triangle}{=} (V_\pi(i),i\in S)^T$ a.s. with $V_\pi(i_0)=\lambda_\pi$, the proof can be found in \cite{BorkarC}. 

We next describe the actor update in the tabular case. Note that in the policy gradient theorem \eqref{PG}, we have summation over $\tilde{d}_{\pi}(i),i \in S$, which is the steady state distribution of the modified Markov chain $\{\tilde{X_n}\}$. This is different from the policy gradient result for average cost criteria (see \cite{SuttonBarto2018}), where the summation is over the stationary distribution of the original Markov chain. Averaging the samples generated from the original simulation will not give us the right gradient. To get around this problem, we use the relative value iteration scheme described in \cite{Abounadi}. Let us define the quantities from the policy gradient \eqref{PG} taking $\xi(i)=1, \forall i\in S$:
\[
\begin{split}
&g(i,a,j):=\nabla_{\theta}\log{\pi(i,a)}\left(\frac{e^{\alpha c(i,a,j)}V_\pi(j)}{V_\pi(i)\lambda_\pi}-1\right),\\
&\rho(i,a,j):=\left(\frac{e^{\alpha c(i,a,j)}V_\pi(j)}{V_\pi(i)\lambda_\pi}\right).
\end{split}
\]
Then we get,
\[
\nabla_\theta \log \lambda_\pi=\sum_{i \in S} \tilde{d}_\pi(i)\sum_{a \in A} \pi(i,a)\sum_{j \in S}p(i,a,j)g(i,a,j).
\]
It can be seen from \eqref{PG} that the gradient of the risk sensitive cost is the expected value of $g(i, a,j)$ but with respect to the stationary measure $\tilde{d}_\pi$ of the twisted kernel $\tilde{p}_\pi$. The relationship between the stationary measures of the original and twisted kernels is not apparent and we have access to samples from the original kernel $P_\pi$ only.

Since $\nabla_\theta \log \lambda_\pi$ has an average cost interpretation, there exists $\tilde{W}(i),i \in S$ such that:
\[
\begin{split}
\tilde{W}(i)+\nabla_\theta \log\lambda_\theta=\sum_{a \in A} \pi(i,a)\sum_{j \in S}p(i,a,j)(g(i,a,j)\\
+\rho(i,a,j)\tilde{W}(j)).
\end{split}
\]
Before we proceed further, we recall the learning scheme used in \cite{Abounadi}. Subsequently, we shall incorporate the same in our algorithm. We define the average cost under any policy $\pi$ as:
\begin{equation}
\begin{split}
G_{\pi}=\sum_{i \in S} d_{\pi}(i)\sum_{a \in A}\pi(i,a)\sum_{j \in S}p(i,a,j)g(i,a,j)
\end{split}
\end{equation}
where $g:S\times A\times S\rightarrow \mathbb{R}$ is the single-stage cost of the associated Markov chain. The relative value iteration (RVI) algorithm when the model is known is given by:
\[
\begin{split}
W_{n+1}(i)=\sum_{a \in A} \pi (i,a) \sum_{j \in S}p(i,a,j)\big(g(i,a,j)+W_n(j)\big)\\
-W_n(i_0).
\end{split}
\]
The online RL version of the RVI update (when the model is unknown) is then given by:
\begin{equation}
\label{avg12}
\begin{split}
W_{n+1}(i_n)=W_n(i_n)+b(\nu(i_n,n))\big(g(i_n,z_n,i_{n+1})\\
+W_n(i_{n+1})-W_n(i_n)-W_n(i_0)\big).
\end{split}
\end{equation}
Then $W_n(i_0)\to G_{\pi}$ a.s.\ as $n\rightarrow\infty$ \cite{Abounadi}. As explained before, averaging the samples obtained from the original simulation does not give the precise gradient and one needs to resort to an off-policy approach. Thus, consider the following algorithm in place of (\ref{avg12}): 
\begin{equation}\label{avg-cost-td}
\begin{split}
\tilde{W}_{n+1}(i_n)=\tilde{W}_n(i_n)+b(\nu(i_n,n))\big(g_n(i_n,z_n,i_{n+1})\\
+\rho_n(i_n,z_n,i_{n+1})\tilde{W}_n(i_{n+1})-\tilde{W}_n(i_n)-\tilde{W}_n(i_0)\big),
\end{split}
\end{equation}
where $\rho_n(i_n,z_n,i_{n+1})$ is an importance sampling (IS) ratio that accounts for the correction that would give the precise gradient. We define below the IS ratio $\rho_n$  and the single-stage cost function $g_n$ in our context. Define
\begin{equation}\label{modified-avg}
\tilde{G}(\pi)=\sum_{i \in S}\tilde{d}_{\pi}(i)\sum_{a \in A}\pi(i,a)\sum_{j \in S}p(i,a,j)g(i,a,j).
\end{equation}
We show in Theorem~\ref{thm2} in the next section that the sequence $\tilde{W}_n(i_0), n\geq 0$ converges to $\tilde{G}(\pi)$ almost surely. With  importance sampling, we get the summation over the stationary distribution of a different Markov chain. Define the cost-per-stage and the importance sampling ratio for our problem as follows:
\[
\begin{split}
&g_n(i_n,z_n,i_{n+1}):=\nabla_{\theta}\log{\pi(i_n,z_n)}\\
&\times \left(\frac{e^{\alpha c(i_n,z_n,i_{n+1})}V_n(i_{n+1})}{V_n(i_n)V_n(i_0)}-1\right),\\
&\rho_n(i_n,z_n,i_{n+1}):=\left(\frac{e^{\alpha c(i_n,z_n,i_{n+1})}V_n(i_{n+1})}{V_n(i_n)V_n(i_0)}\right).
\end{split}
\]
Note that our single-stage cost is a $\mathbb{R}^{x_1}$ vector, and so are the updates $\tilde{W}_n(i)$. Now $V_n\to V_\pi$ and $V_n(i_0) \to \lambda_\pi$, hence $g_n \to g$ and $\rho_n \to \rho$. Therefore, $\tilde{W}_n(i_0)$ converges to $\nabla_{\theta}\log{\lambda_{\pi}}$ in \eqref{PG} almost surely. The proof is an easy extension of Theorem~\ref{thm2} (see also \cite{BorkarBridge}). Collecting samples $g_n$ under the nominal kernel $p$ and using the TD learning algorithm in \eqref{avg-cost-td} with $\rho_n$ as the importance sampling ratio, one can circumvent the twisted kernel issue to evaluate $\nabla_\theta \log \lambda_\pi$. Finally, we have our actor update:
\begin{equation}\label{actor-tabular}
\theta_{n+1}=\theta_n-c(n)\tilde{W}_n(i_0).
\end{equation}
The above method, however, suffers from the curse of dimensionality, which brings us to the function approximation version of the critic $V_n$. We use the critic recursions given in \cite{BasuA}. Let us define,
\[
\begin{split}
&A_n=\sum_{m=0}^ne^{\alpha c(i_m,z_m,i_{m+1})}\phi(i_m)\phi(i_{m+1})^T,\\
&B_n=\sum_{m=0}^n\phi(i_m)\phi(i_m)^T.
\end{split}
\]
We can calculate $A_n$ and $B_n^{-1}$ recursively as follows: 
\begin{equation}\label{critic-quant}
\begin{split}
A_{n+1}=A_n+e^{\alpha c(i_{n+1},z_{n+1},i_{n+2})}\phi(i_{n+1})\phi(i_{n+2})^T,\\
B_{n+1}^{-1}=B_n^{-1}-\frac{B_n^{-1}\phi(i_{n+1})\phi(i_{n+1})^TB_n^{-1}}{1+\phi(i_{n+1})^T B_n^{-1} \phi(i_{n+1})}.
\end{split}
\end{equation}
The second recursion in (\ref{critic-quant}) follows from a 
simple application of the Sherman-Morrison formula and is a computationally efficient technique to recursively compute the inverse of matrix $B_{n+1}$ using the inverse of matrix $B_n$ \cite{NAC}.

To avoid the small divisor problem, we take two small sensitivity parameters $\delta_1,\delta_2>0$. The critic recursion in the function approximation setting is given as follows: 
\begin{equation}\label{critic-r}
r_{n+1}=r_n+a(n)\Big(\frac{B_n^{-1}A_n}{\phi(i_0)^Tr_n \vee \delta_1}-I\Big)r_n.
\end{equation}
We also need to avoid this problem for the importance sampling ratio obtained using the function approximators, which is now given by
\[
\tilde{\rho}(i_n,z_n,i_{n+1})=\frac{e^{\alpha c(i_n,z_n,i_{n+1})}r_n^T\phi(i_{n+1})}{(r_n^T\phi(i_n)r_n^T\phi(i_0))\vee \delta_2}.
\]
We approximate $\tilde{W}_n(i) \approx \tilde{w}\psi(i)$, where $\tilde{w}\in \mathbb{R}^{x_1 \times x_3}$ and $\psi(i)$ are suitable feature vectors. Unlike the tabular off-policy TD recursion \eqref{avg-cost-td}, off-policy TD with function approximation can diverge (see \cite{baird1995}). Hence we adapt the GTD2 algorithm given in \cite{sutton2009}, which is stable in the off-policy setting, see Theorem~\ref{thm3} of section~\ref{convergence} for details. We thus redefine the TD error via the IS ratio $\tilde{\rho}$ as follows: 
\begin{equation}\label{TD-error-2}
\begin{split}
\delta_n=&(\tilde{\rho}(i_n,z_n,i_{n+1})-1)\nabla_{\theta}\log{\pi(i_n,z_n)}- \tilde{w}_n\psi(i_0)\\
&+\tilde{\rho}(i_n,z_n,i_{n+1})\tilde{w}_n\psi(i_{n+1}) - \tilde{w}_n\psi(i_n).
\end{split}
\end{equation}
We have an extra parameter matrix $u_n \in \mathbb{R}^{x_1 \times x_3}$ analogous to the extra parameter vector used in the GTD2 algorithm \cite{sutton2009}. Our intermediate update for importance sampling is given as:
\begin{equation}\label{critic-2}
\begin{split}
u_{n+1}=&u_n+b(n)(\delta_n-u_n\psi(i_n))\psi(i_n)^T,\\
\tilde{w}_{n+1} = &\tilde{w}_n+b(n)u_n\psi(i_n)(\psi(i_n)+\psi(i_0)\\
&-\tilde{\rho}(i_n,z_n,i_{n+1})\psi(i_{n+1}))^T.
\end{split}
\end{equation}
Let $C \subset \mathbb{R}^{x_1}$ be a compact and convex set in which the actor parameter $\theta$ takes values. We define a projection operator $\Gamma: \mathbb{R}^{x_1} \rightarrow C$ that projects any $x \in \mathbb{R}^{x_1}$ to the unique nearest point in the set $C$ (uniqueness being assured by $C$ being compact and convex). Finally, we have our actor update using the approximate estimate of the gradient as follows:
\begin{equation}\label{actor-fa}
\theta_{n+1}=\Gamma(\theta_n-c(n)\tilde{w}_n \psi(i_0)). 
\end{equation}

\begin{algorithm}[tb]\label{algorithm1}
\caption{Risk-Sensitive Actor-Critic with Function Approximation (RSACFA)}
\textbf{Input}: step-sizes
\begin{algorithmic}[1]
\STATE Initialize starting state, $r_0$, $u_0$, $\tilde{w}_0$, $\theta_0$.
\FOR{each timestep \(n \geq 0 \)}
\STATE Play action and collect next state and cost
\STATE Update the required quantities \eqref{critic-quant}
\STATE Perform main critic update \eqref{critic-r}
\STATE Define the TD error via \eqref{TD-error-2}
\STATE Perform the importance sampling update \eqref{critic-2}
\STATE Perform the Actor Update \eqref{actor-fa}
\ENDFOR
\end{algorithmic}
\end{algorithm}

\section{Convergence Analysis}
\label{convergence}
Since $b(n)=o(a(n))$ and $c(n)=o(b(n))$, it follows that  $c(n)=o(a(n))$. Thus by standard two time-scale analysis, $\theta_n\approx \theta$, a constant, when analyzing \eqref{critic-r} and \eqref{critic-2}
(see Chapter 8 of \cite{BorkarB}). Let $D_{\theta}$ be the diagonal matrix with stationary distribution $d_{\theta}(i),i \in S$ of the Markov chain $\{X_n\}$ along the diagonal. Define the matrix
\[
\mathcal{M}_{\theta}=\sqrt{D_{\theta}}\Phi(\Phi^TD_{\theta}\Phi)^{-1}\Phi^TD_{\theta}Q_{\theta} \sqrt{D_{\theta}}^{-1}.
\]
It is easy to see that $\mathcal{M}_{\theta}$ is nonnegative, aperiodic, and the submatrix of $\mathcal{M}_{\theta}$ corresponding to $u \in \cup_k\{j|\phi^k(j)>0\}$ is irreducible with rest of the rows and columns $= 0$. Hence $\mathcal{M}_{\theta}$ has a Perron-Frobenius eigenvalue $\gamma_{\theta}>0$ and associated non-negative eigenvector $Y_{\theta}$ such that $Y_{\theta}(u)>0 \ \forall u$. For $\delta_1<\gamma_{\theta}$ we have the following result.
\begin{thm}
For $\theta_n \equiv \theta$, the iterate sequence $r_n,n\geq 0$ in \eqref{critic-r} converges to $r_{\theta}$ a.s.\ as $n\rightarrow\infty$ such that $\gamma_\theta = \phi(i_0)^Tr_{\theta}$ and 
$Y_{\theta}=\sqrt{D_{\theta}}\Phi r_{\theta}$, respectively.
\end{thm}
\begin{proof}
See Section 5 of \cite{BasuA}. For existence of such $r_\theta$, see \cite{BasuA}.
\end{proof}

\begin{thm}\label{thm2}
For $\theta_n \equiv \theta$, the iterate sequence  $\tilde{W}_n(i_0), n\geq 0$ in \eqref{avg-cost-td} converges to $\tilde{G}(\theta)$ in \eqref{modified-avg} almost surely. 
\end{thm}
\begin{proof}
Let $e$ be the vector of all ones. Define the Bellman operator $T^A_{\pi}:\mathbb{R}^{|S|} \rightarrow \mathbb{R}^{|S|}$ for average cost as follows:
\[
\begin{split}
T^A_{\pi}\tilde{W}(i)=\sum_{a \in A}\pi(i,a)\sum_{j \in S}p(i,a,j)\big(g(i,a,j)\\
+\rho(i,a,j)\tilde{W}(j)\big),i \in S. 
\end{split}
\]
Then the Bellman equation $\tilde{W}=T^A_{\pi}\tilde{W}-\lambda' e$, has a unique solution up to an additive constant when $\lambda'=\tilde{G}(\theta)$. The solution is $\infty$ when $\lambda'>\tilde{G}(\theta)$ and $-\infty$ when $\lambda'<\tilde{G}(\theta)$. Therefore, the above Bellman equation has a solution only when $\lambda'=\tilde{G}(\theta)$, and so the Bellman equation $\tilde{W}=T^A_{\pi}\tilde{W}-\tilde{W}(i_0)e$ has a solution only when $\tilde{W}(i_0)=\tilde{G}(\theta)$. Let us define the function  $f(\tilde{W})=T^A_{\pi}\tilde{W}-\tilde{W}(i_0)e-\tilde{W}$. It is easy to see that $f(\cdot)$ is Lipschitz continuous. The recursion \eqref{avg-cost-td} can be rewritten as:
\begin{equation}
    \label{eqn}
\tilde{W}_{n+1}(i_n)=\tilde{W}_{n}(i_n)+b(\nu(i_n,n))\big[f(\tilde{W}_n)(i_n)+M_{n+1}(i_n)\big],
\end{equation}
where $M_{n+1}(i_n), n\geq 0$ is a martingale difference sequence w.r.t.\ the increasing  sigma fields $\mathcal{F}_n=\sigma(\tilde{W}_m,M_m(i_m),i_m,m \leq n)$, $n\geq 0$. Since $\tilde{W}(n)$ is $\mathcal{F}_n$-measurable and all other quantities are bounded, it is easy to see that for some constant $K>0$,
\[
\mathbb{E}[||M_{n+1}(i_n)||^2|\mathcal{F}_n] \leq K(1+||\tilde{W}_n||^2).
\]
By Assumption~\ref{assum1}, all the states $i \in S$ are visited infinitely often. Let $R_\theta$ be the modified probability transition matrix with elements $R_{\theta}(i,j)=\sum_{a \in A}
\pi(i,a)\rho(i,a,j) p(i,a,j)$ for $i,j \in S$. Let $E$ be a $|S|\times |S|$ matrix with the $i_0$th column as $e$ and all other columns having $0$ as entries.
We now verify the conditions in \cite{BorkarE} to prove the stability and convergence of the recursion. 
It then follows from Chapter 6.4 (1) of \cite{BorkarB}
that the ODE associated with (\ref{eqn}) is $\dot{\tilde{W}}(t) =\frac{1}{|S|} f(\tilde{W}(t))$.
 Next define
\[
f_{\infty}(\tilde{W})=\lim_{c \rightarrow \infty}\frac{f(c\tilde{W})}{c}=R_{\theta} \tilde{W} -\tilde{W}-E\tilde{W}=(R_{\theta}-I-E)\tilde{W}. 
\]

Since $R_\theta$ is a non-negative irreducible probability matrix, by the Perron-Frobenius theorem it has principal eigenvalue $1$ corresponding to right eigenvector $e$. Hence $R_{\theta}-I$ has $0$ as the eigenvalue with highest real part corresponding to the right eigenvector $e$. $E$ has the maximum eigenvalue as $1$ corresponding to the eigenvector $e$, and all other eigenvalues are $0$. Therefore the matrix $R_{\theta}-I-E$ has all eigenvalues with negative real parts. 
Thus, the ODE $\dot{\tilde{W}}(t)=\mathbb{L}f_{\infty}(\tilde{W}(t))$ has the origin as its globally asymptotically stable equilibrium. Hence (A1)-(A2) of \cite{BorkarE} are satisfied and the claim follows from Theorems 2.1-2.2 of \cite{BorkarE}.
\end{proof}

Let us now define the importance sampling ratio using $r_{\theta}$:
\begin{equation}\label{is-ratio}
g^1(i,a,j)=\frac{e^{\alpha c(i,a,j)}r_{\theta}^T\phi(j)}{(r_{\theta}^T\phi(i)r_{\theta}^T\phi(i_0))\vee \delta_2}
\end{equation}
Let $\Psi\in \mathbb{R}^{|S|\times x_3}$ be the feature matrix with rows $\psi(i)^T,i \in S$. Let $\tilde{R}_{\theta}$ be the matrix with elements 
$$\tilde{R}_{\theta}(i,j)=\sum_{a \in A}\pi(i,a)p(i,a,j)g^1(i,a,j), \ i,j \in S.$$ 
Let $B^1_{\theta}$ be a matrix with rows 
$$B^1_{\theta}(i)=\sum_{a \in A}\nabla_{\theta}\pi(i,a)^T\sum_{j \in S}(g^1(i,a,j)-1)p(i,a,j),i \in S.$$ 
Let us define the following matrices:
\[
\begin{split}
&C^1=\Psi^T D_{\theta}\Psi,\\
&A^1=\Psi^T D_{\theta}(\tilde{R}_{\theta}-I-E)\Psi,\\
&b^1=\Psi^T D_{\theta} B^1_{\theta}.
\end{split}
\]
\begin{thm}\label{thm3}
For $\theta_n \equiv \theta$, assuming $A^1$ and $C^1$ are non-singular matrices, the iterates $\tilde{w}_n$, $n\geq 0$, in \eqref{critic-2} converge a.s. to $({A^1}^{-1}b^1)^T$.
\end{thm}
\begin{proof}
Let $U_n=[u_n,\tilde{w}_n]^T$ in recursion \eqref{critic-2}. Then the recursion can be written as,
\begin{equation}
\begin{split}\label{critic-2.1}
U_{n+1}=U_n+b(n)\big(\begin{bmatrix}-C^1_n&-A^1_n\\{A^1_n}^T&0\end{bmatrix}U_n+\begin{bmatrix}b^1_n\\0\end{bmatrix}\big)+\epsilon_1^{\theta}
\end{split}
\end{equation}
where,
\[
\begin{split}
&C^1_n=\psi(i_n)\psi(i_n)^T,\\
&A^1_n=\psi(i_n)(\psi(i_n)+\psi(i_0)-g^1(i_n,z_n,i_{n+1})\psi(i_{n+1}))^T,\\
&b^1_n=(g^1(i_n,z_n,i_{n+1})-1)\psi(i_n)\nabla_{\theta}\log{\pi(i_n,z_n)}^T,
\end{split}
\]
and $\epsilon_1^{\theta} = o(1)$ by Theorem 1.
Let us now define 
\[
f^1(U_n)=\begin{bmatrix}-C^1&-A^1\\{A^1}^T&0\end{bmatrix}U_n+\begin{bmatrix}b^1\\0\end{bmatrix}.
\]
Then, 
$f^1$ is linear in $U_n$ and hence Lipschitz continuous. The recursion \eqref{critic-2.1} can be rewritten as:
\begin{equation}
    \label{Ueq}
U_{n+1} = U_n + b(n)\Big(f^1(U_n)+\epsilon^1_n+M^1_{n+1}\Big), 
\end{equation}
where the noise term $M_{n+1}^1$ is defined as:
\[
\begin{split}
M_{n+1}^1=&\begin{bmatrix}-C^1_n&-A^1_n\\{A^1_n}^T&0\end{bmatrix}U_n+\begin{bmatrix}b^1_n\\0\end{bmatrix}\\
&-\mathbb{E}\Big[\begin{bmatrix}-C^1_n&-A^1_n\\{A^1_n}^T&0\end{bmatrix}U_n+\begin{bmatrix}b^1_n\\0\end{bmatrix}\Big|\mathcal{F}^1_n\Big], \mbox{ }\forall n,
\end{split}
\]
where the sigma fields $\mathcal{F}^1_n$ are defined by 
$$\mathcal{F}^1_n= \sigma(u_m, \tilde{w}_m, i_m, a_m, m\leq n), \ n\geq 0.$$
Since $U_n$ is $\mathcal{F}^1_n$-measurable and all other quantities are bounded, it is easy to see that for some constant $K^1>0$,
\[
\mathbb{E}[\|M^1_{n+1}\|^2|\mathcal{F}^1_n] \leq K^1(1+\|U_n\|^2).
\]
In (\ref{Ueq}), the quantity $\epsilon^1_n$ then corresponds to
\[
\epsilon^1_n = \mathbb{E}\Big[\begin{bmatrix}-C^1_n&-A^1_n\\{A^1_n}^T&0\end{bmatrix}U_n+\begin{bmatrix}b^1_n\\0\end{bmatrix}\Big|\mathcal{F}^1_n\Big] - f^1(U_n),
\]
and constitutes the Markov noise component. Note that by ergodicity of the Markov chain, $\epsilon^1_n\rightarrow 0$ almost surely as $n\rightarrow\infty$. 
Now define,
\[
f^1_{\infty}(U)=\lim_{c \rightarrow \infty}\frac{f^1(cU)}{c}=G^1U,
\]
where $G^1=\begin{bmatrix}-C^1&-A^1\\{A^1}^T&0\end{bmatrix}$. We get $det(G^1)=det(A^1)^2\neq 0$. Then all eigenvalues of $G^1$ are nonzero. Let $\lambda \neq 0$ be an eigenvalue of $G^1$ and $v$ be the corresponding normalized eigenvector. Then $v^*G^1v=\lambda$ where $v^*$ is the complex conjugate of $v$. Let $v^T=(v_1^T,v_2^T)$, where $v_1,v_2 \in \mathbb{C}^{x_3}$. Then $Re(\lambda)=(v^*G^1v+(v^*G^1v)^*)/2=-v_1^*C^1v_1=-(\Psi v_1)^*D_{\theta}\Psi v_1\leq 0$. Since $C^1$ is non-singular, $\Psi$ has full rank. We see that $v_1\neq 0$, because if $v_1=0$, then $\lambda=v^*G^1v=0$, which is a contradiction. Therefore $(\Psi v_1)^*D_{\theta}\Psi v_1\neq 0$, hence $Re(\lambda)<0$. Hence the ODE $\dot{U}(t)=f^1_{\infty}(U(t))$ has the origin as its globally asymptotically stable equilibrium. 
By Theorem 8.3 of \cite{BorkarB}, $\sup_n \|\tilde{w}_n\|<\infty$ almost surely. The claim now follows from Theorem 8.2 - Corollary 8.1 of \cite{BorkarB}.
\end{proof}

We now define the directional derivative of $\Gamma(\cdot)$ at point $x$ along the direction $y$ as follows:
\[\bar{\Gamma}(x,y)=\lim_{\eta\downarrow 0}\Big(\frac{\Gamma(x+\eta y)-x}{\eta}\Big).\]
Let $\kappa\stackrel{\triangle}{=} \{\theta|\bar{\Gamma}(\theta,\nabla_{\theta}\log{\lambda_{\theta}})=0\}$. This is the set of $\theta$ such that the policy gradient in \eqref{PG} is 0 or at a boundary of the projection set of $\Gamma$, which means $\theta$ is at a bounded local optimum.  Let $\kappa^{\epsilon}$ be the $\epsilon$-neighborhood of $\kappa$.
\begin{thm}\label{thm4}
Given $\epsilon>0$, $\exists \delta>0$ such that if $\sup_{\theta}||\mathcal{E}^{\theta}||<\delta$ then the actor update in \eqref{actor-fa} converges almost surely to the set $\kappa^{\epsilon}$.
\end{thm}
\begin{proof}
By Theorem~\ref{thm3}, the recursion \eqref{actor-fa} can be rewritten as:
\[
\theta_{n+1}=\Gamma(\theta_n-c(n)({A^1}^{-1}b^1)^T\psi(i_0)+\epsilon_2^{\theta_n}),
\]
where $\epsilon_2^{\theta_n} \rightarrow 0$. By Assumption~\ref{assum2} and Theorem 2 of \cite{Schweitzer1968}, $d_{\theta}$ is continuously differentiable in $\theta$. Since $\theta_n$ lies in a compact space, $\nabla_{\theta}\mathcal{M}_{\theta}$ is uniformly bounded, hence $\mathcal{M}_{\theta}$ is Lipschitz. Therefore its eigenvalue $\gamma_{\theta}$ is continuous in $\theta$. Let $A_n$ be a sequence of non-negative, irreducible and aperiodic matrices with principal eigenvalue $\lambda_n>0$ and eigenvector $v_n>0$ such that $v_n^T\phi(s_0)=\lambda_n$. Therefore $A_nv_n=\lambda_nv_n$. Now suppose $A_n \rightarrow A$ as $n\rightarrow\infty$. Then $\lambda_n \rightarrow \lambda$, where $\lambda$ is the principal eigenvalue of $A$. Thus, $v_n \rightarrow v$ where $v$ is the principal eigenvector of $A$ such that $v^T\phi(s_0)=\lambda$. Therefore, $g^1(i,a,j)$ in  \eqref{is-ratio} is continuous in $\theta$. Therefore by Assumption~\ref{assum2}, $\tilde{R}_{\theta}$ and $B^1_{\theta}$ are also continuous. Thus ${A^1}^{-1}b^1$ is continuous in $\theta$. Consider the associated ODEs:
\begin{equation}\label{risk_actor_ode1}
\dot{\theta}=\bar{\Gamma}(\theta,({A^1}^{-1}b^1)^T\psi(i_0))
\end{equation}
\begin{equation}\label{risk_actor_ode2}
\dot{\theta}=\bar{\Gamma}(\theta,\nabla_{\theta}\log{\lambda_{\theta}})
\end{equation}
Therefore by Theorem 5.3.1 of \cite{kushner1978}, \(\theta(n)\) asymptotically tracks the ODE \eqref{risk_actor_ode1}. By Assumption~\ref{assum2}, $Q_\theta$ is continuously differentiable. Hence by \cite{deutsch1984}, $\lambda_\theta$ is continuously differentiable and by \cite{deutsch1985}, $V_\theta$ is continuously differentiable. As $\theta_n$ lies in a compact space, $\lambda_\theta$ and $V_\theta(i)$ is bounded below, hence $\tilde{p}_\theta$ is continuously differentiable. Hence by \cite{Schweitzer1968}, $\tilde{d}_\theta$ is continously differentiable. Now by Assumption~\ref{assum2}, $\nabla_{\theta}\log{\lambda_{\theta}}$ is continuosly differentiable. Since $\theta_n$ lies in a compact space, $\nabla^2_{\theta}\log{\lambda_{\theta}}$ is uniformly bounded, hence $\nabla_{\theta}\log{\lambda_{\theta}}$ is Lipschitz. If $\sup_\theta||\mathcal{E}^{\theta}||<\delta$ then, $\theta(n)$ is a perturbed trajectory of \eqref{risk_actor_ode2}. Now $\kappa$ is the set of stable equilibria of \eqref{risk_actor_ode2}. The claim follows from Theorem 1 of \cite{hirsch1989}.
\end{proof}
Hence, by Theorem 1-4, the algorithm RSACFA converges asymptotically.
\section{Numerical Results}
\label{numerical}

\begin{figure*}
\centering
\includegraphics[width=1\textwidth]{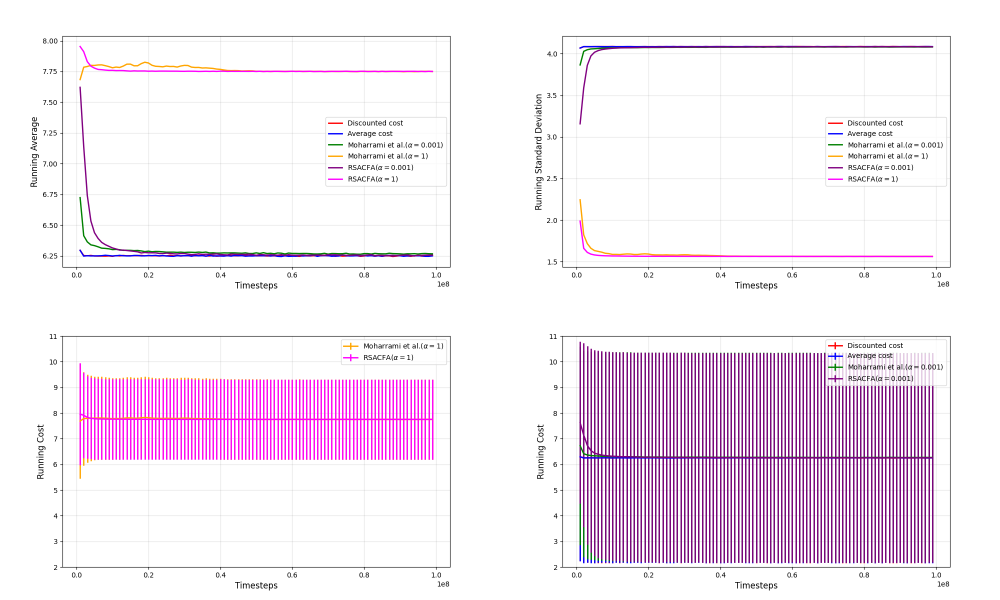}
\caption{State space size $|S|=9$, Action space size $|A|=9$}
\label{fig1}
\end{figure*}

\begin{figure*}
\centering
\includegraphics[width=1\textwidth]{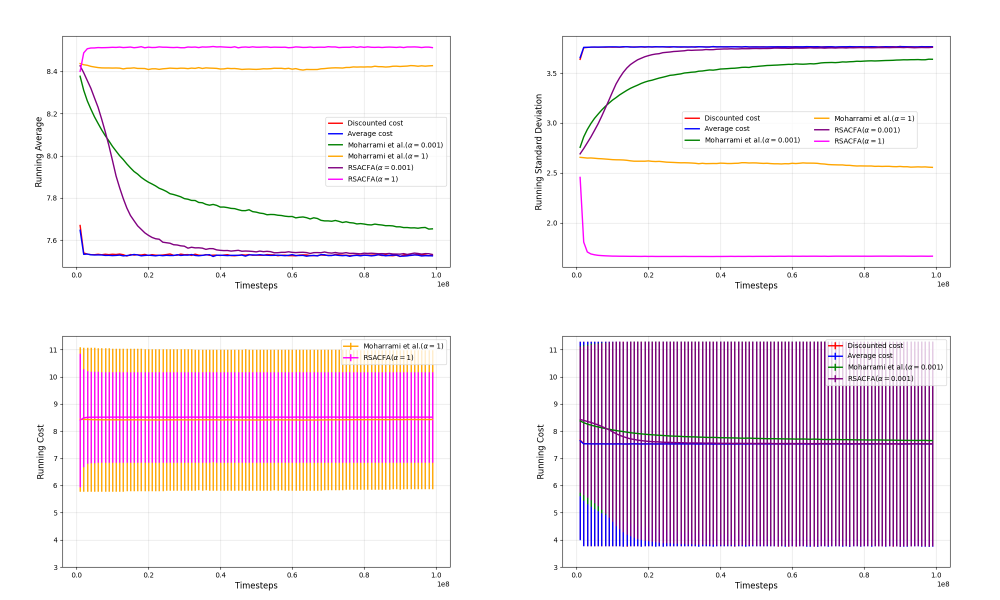}
\caption{State space size $|S|=100$, Action space size $|A|=9$}
\label{fig2}
\end{figure*}

{\small
\begin{table*}
\centering
\begin{tabular}{|p{0.8cm}|p{2.4cm}|p{3cm}|p{2.4cm}|p{3cm}|}
\hline
Exp. & RSACFA & Moharrami et al. & Average Cost & Discounted Cost\\
\hline
1 & $11045.49$ $\pm96.16$ & $20714.60$ $\pm163.69$ &$7538.60$ $\pm118.94$ & $7403.59$ $\pm23.44$\\
\hline
2 & $11613.45$ $\pm101.78$ & $20618.56$ $\pm139.81$ &$7464.00$ $\pm41.76$ & $7412.19$ $\pm66.92$\\
\hline
3 & $33152.86$ $\pm454.21$ & $49596.65$ $\pm181.00$ &$17894.83$ $\pm739.79$ & $17827.96$ $\pm389.16$\\
\hline
\end{tabular}
\caption{Mean and standard deviation($\mu \pm \sigma$) of computational time(in secs)\\ after \(10^8\) time-steps}
\label{table1}
\end{table*}
}

\begin{figure*}
\centering
\includegraphics[width=1\textwidth]{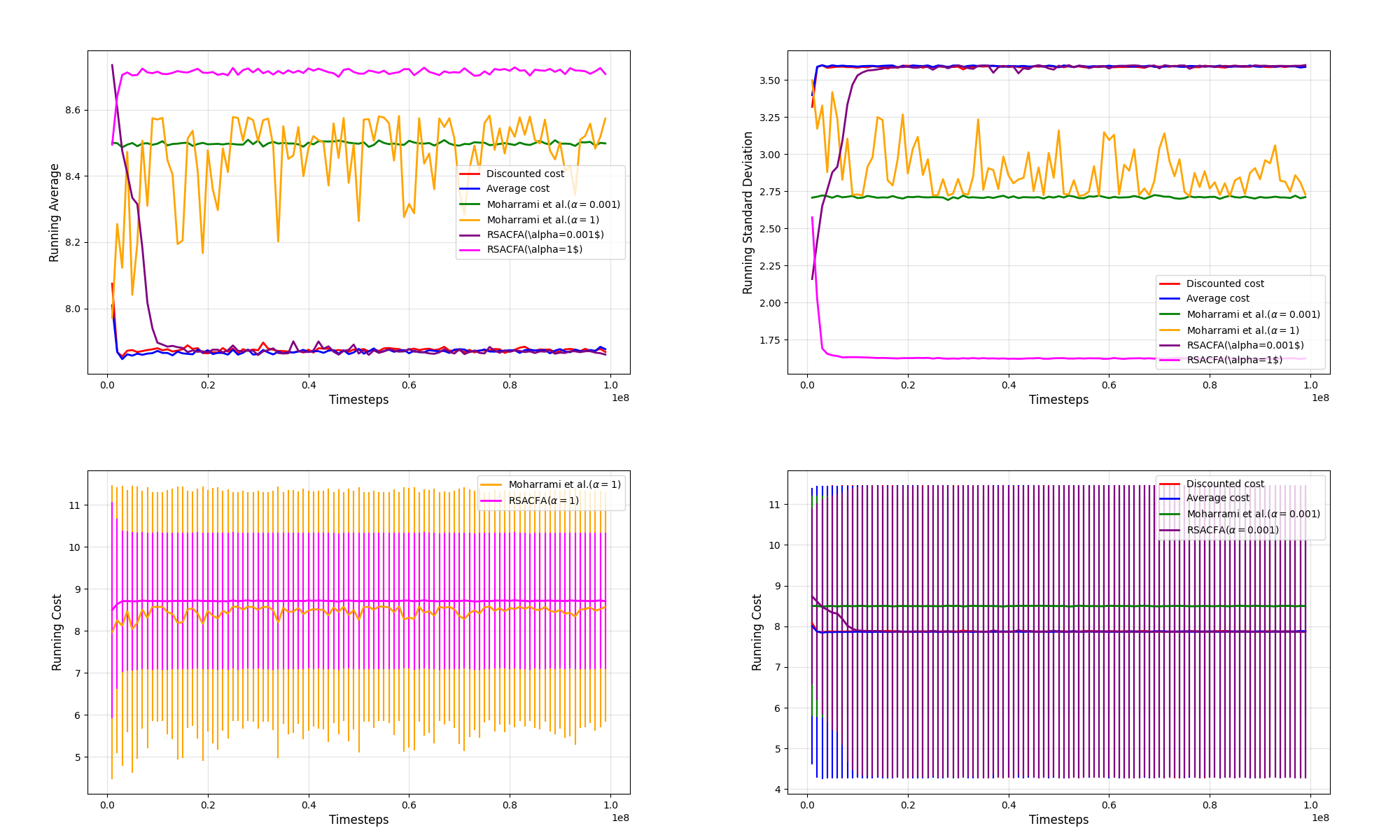}
\caption{State space size $|S|=10000$, Action space size $|A|=9$}
\label{fig3}
\end{figure*}

\begin{figure*}
\centering
\includegraphics[width=1\textwidth]{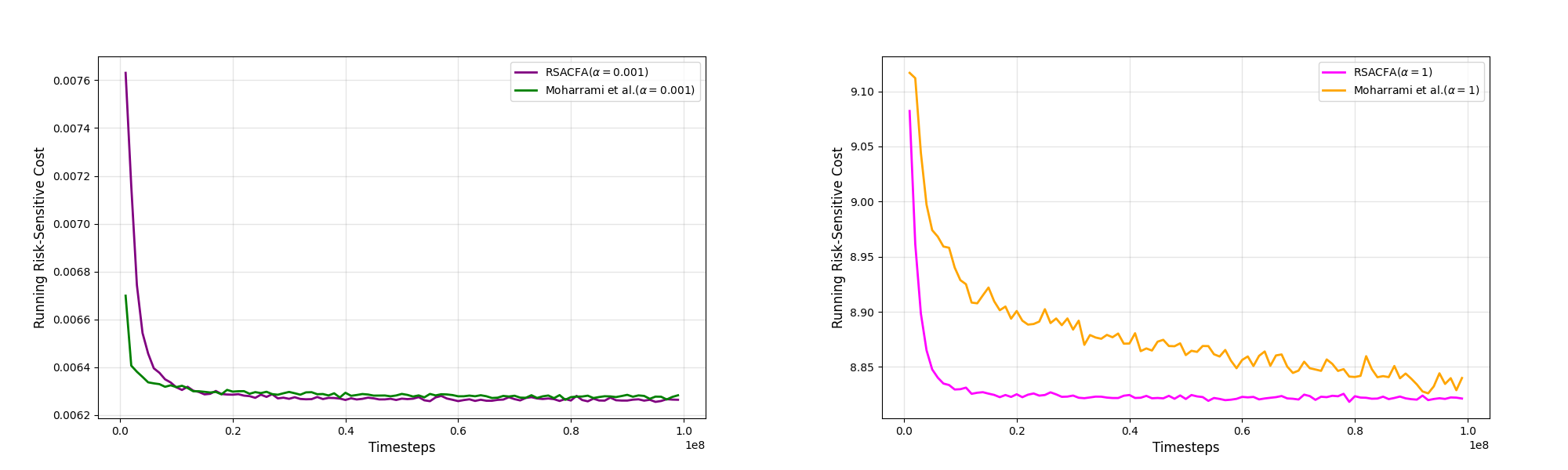}
\caption{State space size $|S|=9$, Action space size $|A|=9$}
\label{fig4}
\end{figure*}

\begin{figure*}
\centering
\includegraphics[width=1\textwidth]{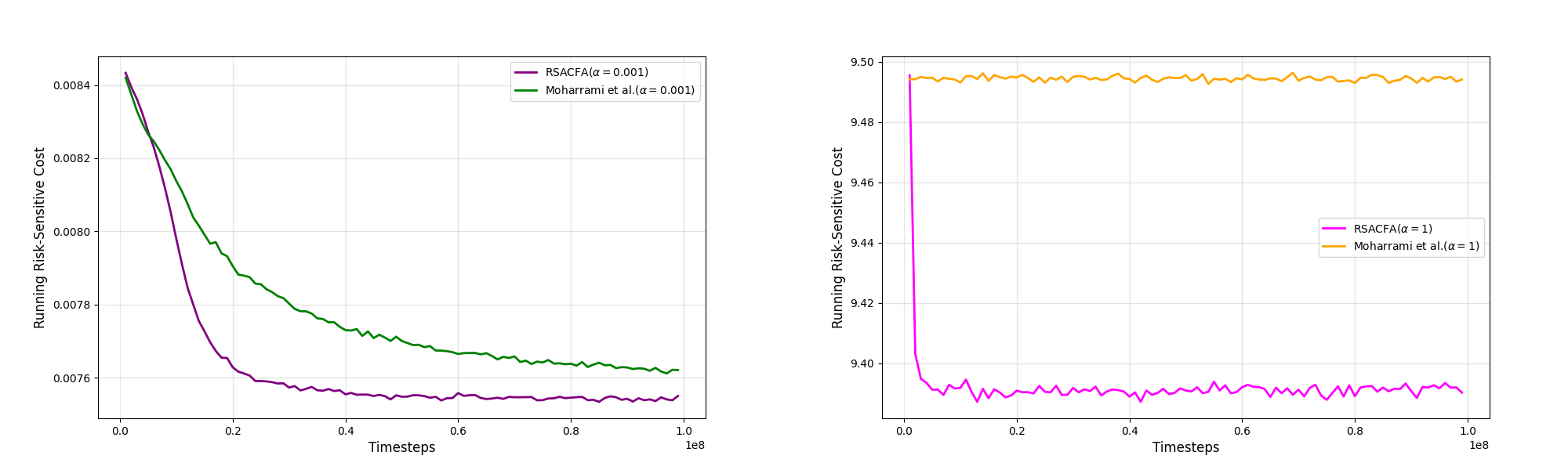}
\caption{State space size $|S|=100$, Action space size $|A|=9$}
\label{fig5}
\end{figure*}

\begin{figure*}
\centering
\includegraphics[width=1\textwidth]{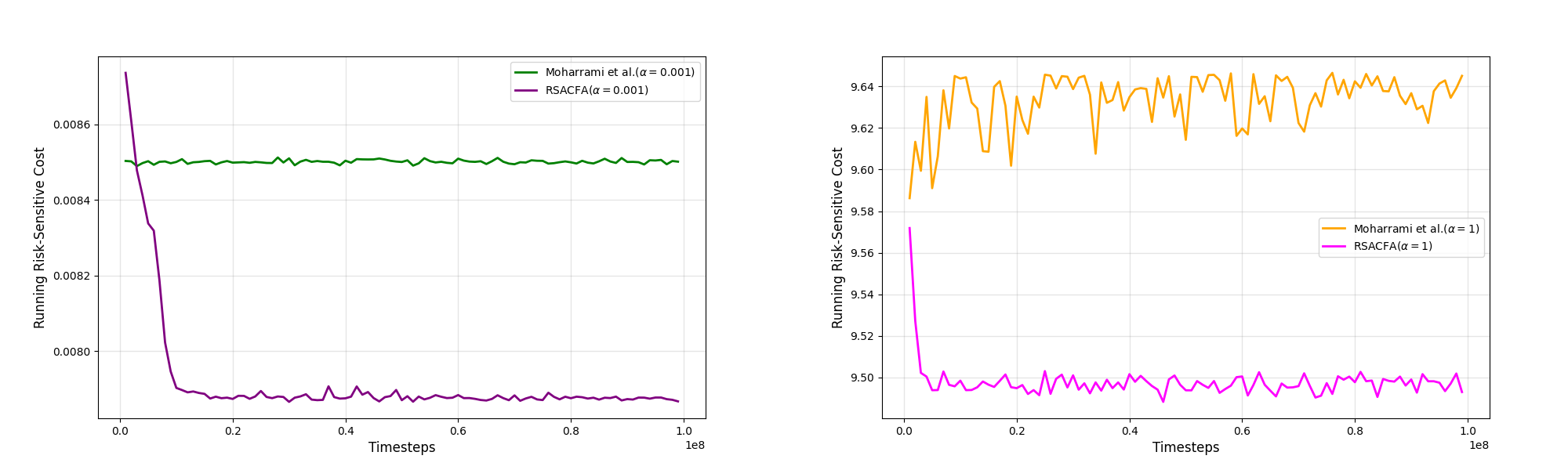}
\caption{State space size $|S|=10000$, Action space size $|A|=9$}
\label{fig6}
\end{figure*}

We compare our risk-sensitive policy gradient algorithm\footnote{The code is available at https://github.com/gsoumyajit/Risk-Sensitive-PG-with-Function-Approximation} with the algorithm in \cite{moharrami2024} and risk-sensitive analogs of other policy gradient algorithms such as for the average cost setting and discounted cost setting \cite{KondaT}, respectively, where the critic update used is given in \cite{TsitsiklisV}. We show the plots of mean and standard deviation resp., of the running costs in order to compare the performance of these algorithms. We observe that with high enough risk factor, our algorithm successfully finds policies with low standard deviation. This is in contrast to average or discounted cost settings, where one may find policies with a high standard deviation. In our experiments, the policy for every algorithm starts from the same point. We also used the same policy features for all the algorithms.

In each main figure, the upper left sub-figure plots the running average while the upper right sub-figure plots the running standard deviation, respectively, of the costs achieved by the different algorithms. To make the comparison clear, the bottom two sub-figures in each main figure plot the mean($\mu$) as the curve with the standard deviation($\sigma$) as the error bars from $\mu+\sigma$ to $\mu-\sigma$.
In each main figure, the bottom left sub-figure plots the comparison between RSACFA and \cite{moharrami2024} for risk factor $\alpha=1$ and the bottom right sub-figure plots the comparison between average cost, discounted cost, RSACFA and \cite{moharrami2024} for risk factor $\alpha=0.001$.
The experiments are run on a two-dimensional grid world setting. Along each dimension, the agent can go forward, go backward, or stay in the same position. Thus, in this setting, there are total of a $3^2=9$ actions numbered $0-8$ in each state. These can be listed as `go left', `go right', `go up', `go down', `go top-left', `go top-right', `go bottom-left', `go bottom-right' and `stay put', respectively. The effect of taking any of these actions is as follows: There is a $0.5$ probability that the agent will go one step in the direction suggested by the action and $0.5$ probability that it will go in any of the other directions.

There are some states with which a fixed cost of $10$ is associated, and with other states the cost is decided by the following rule: If the action is even valued in the aforementioned set of feasible actions, and if the direction is the same as the one suggested by the action, then the single-stage cost is $6$, else the single-stage cost is $8$. Hence, the expected value and standard deviation of the single-stage cost of taking an even action are $7$ and $1$, respectively. If, on the other hand, the action is odd-valued, and the direction of movement is the same as that suggested by the action, then a cost of $1$ is obtained, else we incur a cost of $9$. Hence, the expected value and standard deviation of taking an odd action are $5$ and $4$, respectively.

For Figure~\ref{fig1}, we ran experiments on a $3\times3$ grid ($9$ states). From the bottom sub-figures of Figure~\ref{fig1}, we can see our algorithm for risk factor $\alpha=1$ is more robust than average or discounted cost algorithms. This is because the highest point in our algorithm for $\alpha=1$ is lower than the highest point of average or discounted cost algorithms, making our policy safer. Our algorithm was able to find a policy which gave running costs with similar average and similar standard deviation to that of \cite{moharrami2024} for both risk factor $\alpha=1$ and $\alpha=0.001$. For a low risk factor of $\alpha=0.001$, our algorithm was able to give lower average but higher standard deviation than that with $\alpha=1$ and that is  similar to the average and discounted cost settings. This is expected because we know that as $\alpha \rightarrow 0$, the risk-sensitive problem tends towards an average cost problem.

For Figure~\ref{fig2}, we experimented on a $10\times10$ grid ($100$ states).
For $\alpha=1$, our algorithm was able to learn a policy which gave running costs with lower standard deviation than that of \cite{moharrami2024}. Also from the bottom sub-figures of Figure~\ref{fig2}, we can see that the highest point of our algorithm for $\alpha=1$ is lower than the highest point of that of \cite{moharrami2024}, making our policy more robust. Thus, our algorithm reduces iterate variance better. We can also see that our algorithm gives a more robust policy than the average and discounted-cost settings similar to Figure~\ref{fig1}. In Figure~\ref{fig2}, for $\alpha=0.001$, our algorithm (as expected) was able to give running costs with a similar average and standard deviation to the average and discounted cost algorithms. However, unlike our algorithm, the one in \cite{moharrami2024} did not come as close to these costs due to its slow learning resulting from higher iterate variance in their case. 
We therefore observe that our algorithm performs better than the recent policy gradient algorithm for risk-sensitive cost presented in \cite{moharrami2024} for large state spaces. Also, the algorithm in \cite{moharrami2024} takes more computational time than our algorithm RSACFA. To show an empirical comparison, we provide in Table~\ref{table1} the computational time taken by the different algorithms. It can be seen that RSACFA is better than \cite{moharrami2024} on all settings that we consider.

For Figure~\ref{fig3}, we experimented on a $100\times100$ grid with $10,000$ states. For $\alpha=1$, our algorithm was able to learn policy which gave running costs with lower standard deviation than that of \cite{moharrami2024} as well as for the average and discounted-cost criteria. From the bottom sub-figures of Figure~\ref{fig3}, we can see that the highest point of our algorithm for $\alpha=1$ is lower than the highest point of that of \cite{moharrami2024}, as well as that obtained for  average and  discounted-cost criteria, making our algorithm more robust. In Figure~\ref{fig3}, for $\alpha=0.001$, our algorithm (as expected) was able to find policies giving running costs with similar average and standard deviation as the average and discounted-cost criteria. Unlike RSACFA, the algorithm of \cite{moharrami2024} did not minimize the average or the standard deviation of the running costs. We thus observe that our algorithm shows good performance even for large state spaces unlike the algorithm of \cite{moharrami2024}, whose performance is not up to the mark in such settings.

In Figures~\ref{fig4}--\ref{fig6}, we plot the running risk-sensitive cost for the above experiments involving $9$, $100$ and $10,000$ states, respectively. The running risk-sensitive cost is calculated by the formula: $\log\left(\frac{\sum_{i=1}^N \exp(\alpha c_i)}{N}\right)$, where $c_1,\ldots, c_N$ are the single-stage costs over the last $N$ stages. In each figure, the left sub-figure corresponds to $\alpha=0.001$, while the right sub-figure corresponds to the case of $\alpha=1$. We can see that our algorithm successfully reduces the running risk-sensitive cost in all the figures. However, the algorithm in \cite{moharrami2024} was not able to reduce the risk-sensitive cost in the second sub-figure of Figure~\ref{fig5} and for both sub-figures of Figure~\ref{fig6}. This makes our algorithm more preferable to that of \cite{moharrami2024}.

\section{Conclusions}
\label{conclusion}

We presented a policy gradient algorithm for the risk-sensitive cost criterion. Our algorithm incorporates  function approximation, and gives better performance than other algorithms in the literature.
Note that unlike \cite{moharrami2024}, we do not try to estimate the risk-sensitive cost. Also, because of the form of our critic update rule (\ref{critic-2}) where the exponentiated cost in the importance sampling ratio term is multiplied by a ratio of value estimates, it can be seen that the recursions remain stable asymptotically. The actor recursion, on the other hand, remains stable due to the projection to a compact set.
An important direction would be to develop natural actor-critic algorithms such as \cite{NAC} for risk-sensitive cost as they can potentially show improved convergence behavior.


\begin{thebibliography}{99}

\bibitem{Abounadi} J. Abounadi, D. P. Bertsekas, and V. S. Borkar, "Learning Algorithms for Markov Decision Processes with Average Cost", SIAM J. Control. Optim., 40, 681-698, 2001.

\bibitem{artzner} P. Artzner, F. Delbaen, J.M. Eber, and D. Heath, "Coherent measures of risk", Mathematical finance, 9(3), 203-228, 1999.

\bibitem{baird1995} L. C. Baird, "Residual algorithms: Reinforcement learning with function approximation." In Machine learning proceedings 1995, pp. 30-37. Morgan Kaufmann, 1995.

\bibitem{BasuA} A. Basu, T. Bhattacharyya, and V. S. Borkar, "A Learning Algorithm for Risk-Sensitive Cost". Mathematics of Operations Research, 33(4), 880-898, 2008.

\bibitem{NDPbook} D. P. Bertsekas and J. N. Tsitsiklis, Neurodynamic Programming. Athena Scientific, Belmont, MA, 1996.

\bibitem{Bhatnagaretal} S. Bhatnagar, M. C. Fu, S. I. Marcus and I.-J. Wang, "Two-timescale simultaneous perturbation stochastic approximation using deterministic perturbation sequences". ACM Transactions on Modeling and Computer Simulation, 13(2), 180-209, 2003.

\bibitem{BhatnagarBook} S. Bhatnagar, H. L. Prasad and L. A. Prashanth, {\em Stochastic Recursive Algorithms for Optimization: Simultaneous Perturbation Methods}, Lecture Notes in Control and Information Sciences Series, Vol. 434, 2013. 

\bibitem{JMLR} S. Bhatnagar, V. S. Borkar and  A. Madhukar, "A Simulation-Based Algorithm for Ergodic Control of Markov Chains Conditioned on Rare Events". Journal of Machine Learning Research, 7(10), 1937-1962, 2006.

\bibitem{NAC} S. Bhatnagar, R.S. Sutton, M. Ghavamzadeh and M. Lee, "Natural actor–critic algorithms". Automatica, 45(11), 2471-2482, 2009.

\bibitem{BorkarAsyn} V. S. Borkar, "Asynchronous stochastic approximations". SIAM Journal on Control and Optimization 36(3), 840-851, 1998.

\bibitem{BorkarD} V. S. Borkar, "A sensitivity formula for risk-sensitive cost and the actor-critic algorithm", Systems \& Control Letters, Volume 44, Issue 5, Pages 339-346, 2001.

\bibitem{BorkarC} V. S. Borkar, "Q-Learning for Risk-Sensitive Control". Mathematics of Operations Research, 27(2), 294-311, 2002.

\bibitem{BorkarBridge} V. S. Borkar, "Reinforcement Learning—A Bridge Between Numerical Methods and Monte Carlo", Perspectives in mathematical sciences I: probability and statistics, pp. 71-91. 2009.

\bibitem{BorkarB} V. S. Borkar, Stochastic Approximation: A Dynamical Systems Viewpoint (2nd edition), Hindustan Publishing Agency and Springer, 2022.

\bibitem{BorkarM} V. S. Borkar, and S. P. Meyn, "Risk-Sensitive Optimal Control for Markov Decision Processes with Monotone Cost". Mathematics of Operations Research, 27(1), 192-209, 2002.

\bibitem{BorkarE} V. S. Borkar and S. P. Meyn, "The ODE method for convergence of stochastic approximation and reinforcement learning". SIAM Journal on Control and Optimization, 38(2), 447-469, 2000.

\bibitem{M2} R. Cavazos-Cadena, R. Montes-de-Oca, "The Value Iteration Algorithm in Risk-Sensitive Average Markov Decision Chains with Finite State Space". Mathematics of Operations Research 28(4), 752-776, 2003.

\bibitem{deutsch1984} Deutsch, Emeric, and Michael Neumann. "Derivatives of the Perron root at an essentially nonnegative matrix and the group inverse of an M-matrix." Journal of Mathematical Analysis and Applications 102, no. 1: 1-29, 1984.

\bibitem{deutsch1985} Deutsch, Emeric, and Michael Neumann. "On the first and second order derivatives of the Perron vector." Linear algebra and its applications 71: 57-76, 1985.

\bibitem{fei2021} Y. Fei, Z. Yang, Z. Wang, "Risk-sensitive reinforcement learning with function approximation: A debiasing approach." In International Conference on Machine Learning, pp. 3198-3207. PMLR, 2021.

\bibitem{hirsch1989} Hirsch, Morris W. "Convergent activation dynamics in continuous time networks." Neural networks 2, no. 5: 331-349, 1989.

\bibitem{KondaT} V. R. Konda, and J. N. Tsitsiklis,  "On actor-critic algorithms". SIAM journal on Control and Optimization, 42(4), pp.\ 1143-1166, 2003.

\bibitem{kupper} M. Kupper and W. Schachermayer, "Representation results for law invariant time consistent functions", Mathematics and Financial Economics, 2(3), pp.189-210, 2009.

\bibitem{kushner1978} H. J. Kushner and D. S. Clark, Stochastic approximation methods for constrained and unconstrained systems. Springer Science \& Business Media, 2012, vol. 26.

\bibitem{LakshminarayananC} C. Lakshminarayanan, \& S. Bhatnagar, "A stability criterion for two timescale stochastic approximation schemes". Automatica, 79, 108-114, 2017.

\bibitem{moharrami2024} Mehrdad Moharrami, Yashaswini Murthy, Arghyadip Roy, R. Srikant, "A Policy Gradient Algorithm for the Risk-Sensitive Exponential Cost MDP". Mathematics of Operations Research, 2024.

\bibitem{murthy2023} Murthy, Yashaswini, Mehrdad Moharrami, and R. Srikant. "Modified Policy Iteration for Exponential Cost Risk Sensitive MDPs." In Learning for Dynamics and Control Conference, pp. 395-406. PMLR, 2023.

\bibitem{nass2019} D. Nass, B. Belousov and J. Peters, "Entropic Risk Measure in Policy Search," 2019 IEEE/RSJ International Conference on Intelligent Robots and Systems (IROS), Macau, China, pp. 1101-1106, 2019.

\bibitem{Schweitzer1968} P. J. Schweitzer, “Perturbation Theory and Finite Markov Chains.” Journal of Applied Probability, vol. 5, no. 2, pp. 401–13, 1968.

\bibitem{SuttonBarto2018} R. S. Sutton, and A. G. Barto, Reinforcement learning: An introduction. MIT press, 2018.

\bibitem{sutton2009} R. S. Sutton, H. R. Maei, D. Precup, S. Bhatnagar, D. Silver, C. Szepesvári, and E. Wiewiora. "Fast gradient-descent methods for temporal-difference learning with linear function approximation." In Proceedings of the 26th annual international conference on machine learning, pp. 993-1000. 2009.

\bibitem{TsitsiklisV} J. N. Tsitsiklis and B. Van Roy, "An analysis of temporal difference learning with function approximation", IEEE Transactions on Automatic Control,
42(5), pp.\ 674-690, 1997.

\end{thebibliography}
\end{document}